\documentclass{article}

\usepackage{PRIMEarxiv}

\usepackage{plainsty}
\usepackage[utf8]{inputenc} 
\usepackage[T1]{fontenc}    
\usepackage{url}            
\usepackage{booktabs}       
\usepackage{amsfonts}       
\usepackage{nicefrac}       
\usepackage{microtype}      
\usepackage{xcolor}         

\newtheorem{corollary}[theorem]{Corollary}
\newtheorem{remark}{Remark}

\def\SI{\textrm{SI}}

\def\aone{`ABP-M'}
\def\atwo{`ABP-L'}
\def\base{`Mag'}

\title{A Theoretical Understanding of Neural Network Compression from Sparse Linear Approximation}

\author{%
  Wenjing Yang\thanks{Equal contribution.} \\
  School of Statistics\\
  University of Minnesota\\
  Minneapolis, MN 55455 \\
  \texttt{yang2987@umn.edu} \\
   \And
   Ganghua Wang$^*$\\
   School of Statistics\\
  University of Minnesota\\
  Minneapolis, MN 55455 \\
  \texttt{wang9019@umn.edu} \\
   \And
   Jie Ding \\
   School of Statistics\\
  University of Minnesota\\
  Minneapolis, MN 55455 \\
   \texttt{dingj@umn.edu} \\
   \And
   Yuhong Yang \\
   School of Statistics\\
  University of Minnesota\\
  Minneapolis, MN 55455 \\
   \texttt{yangx374@umn.edu} \\
}

\begin{document}

\maketitle

\begin{abstract}
The goal of model compression is to reduce the size of a large neural network while retaining a comparable performance. As a result, computation and memory costs in resource-limited applications may be significantly reduced by dropping redundant weights, neurons, or layers. There have been many model compression algorithms proposed that provide impressive empirical success. However, a theoretical understanding of model compression is still limited. One problem is understanding if a network is more compressible than another of the same structure. Another problem is quantifying how much one can prune a network with theoretically guaranteed accuracy degradation. In this work, we propose to use the sparsity-sensitive $\ell_q$-norm ($0<q<1$) to characterize compressibility and provide a relationship between soft sparsity of the weights in the network and the degree of compression with a controlled accuracy degradation bound. We also develop adaptive algorithms for pruning each neuron in the network informed by our theory. Numerical studies demonstrate the promising performance of the proposed methods compared with standard pruning algorithms.
\end{abstract}

\section{Introduction}
    Since the recent revival of neural networks by deep learning~\citep{lecun2015deep}, the approach of over-parameterized network has achieved a huge success in a wide range of areas such as computer vision~\citep{krizhevsky2012imagenet, he2016identity, redmon2016you} and nature language processing~\citep{devlin2018bert, radford2018improving}. The size of neural networks has grown enormously. For example, LeNet-5 model~\citep{lecun1998gradient} for digit classification in 1998 had 60 thousand parameters, while in 2020, the GPT-3 language model~\citep{brown2020language} has 175 billion parameters, requiring a vast amount of computation, storage, and energy resource for training and prediction. The large model size may also prohibit the deployment to edge devices such as cell phones due to hardware constraints. 
    
    Interestingly, researchers have empirically shown that one can often build a much simpler network with similar performance based on a pre-trained network, which
    can save a considerable amount of resources~\citep{han2015deep, frankle2018lottery}. For example, \citet{han2015deep} show that on the ImageNet dataset, AlexNet can be compressed to retain only $3\%$ of the original parameters without impacting classification accuracy.
    As a result, a rising interest is to compress an over-parameterized neural network to a simpler model with comparable prediction accuracy.
    
    The study of neural network compression has a long history and various methods have been proposed. 
    One of the most popular and effective ways to obtain a simplified model is by pruning the original network~\citep{lecun1989optimal, hagiwara1993removal, han2015deep, hu2016network, luo2017thinet, frankle2018lottery, lee2018snip, he2017channel}, which means dropping some neuron connections, neurons, or neuron-like structures such as filters and layers. The general intuition is that in an over-parameterized model, many connections are redundant and the removal of them will have little impact on prediction. 
    To prune a pre-trained network, a standard procedure consists of a pruning criterion and a stopping criterion. A number of pruning criteria have been proposed from different interpretations of redundancy. For example, one may postulate that weights with small magnitudes are non-influential and thus can be eliminated~\citep{hagiwara1993removal, han2015deep}, or one may prune the edges/neurons that have the least influence on the network output~\citep{lecun1989optimal, lee2018snip,hu2016network,soltani2021information}.  A comprehensive survey of the existing literature on model compression can be found in~\citep{hoefler2021sparsity}. 
    A commonly used stopping criterion is to terminate pruning when the test accuracy on the validation dataset starts to drop significantly.

    Most existing pruning algorithms are heuristic. 
    In practice, the compressibility of different networks and tasks may vary, the efficiency of pruning methods may be unpredictable, and a lot of ad hoc fine-tuning may be involved during pruning. Thus, it is critical to develop a theoretical understanding of compressibility. There have been some recent works in this regard. 
    For example, \citet{arora2018stronger} relate the compressibility and generalization error of a model through its noise stability. They show that the stability of a neural network against noise injection in the input implies the existence of a smaller network with similar performance. 
    Another line towards the provable existence of a sparser network with similar performance is to look for a coreset of parameters. A coreset means a small subset of parameters that can preserve the output of the original network. \citet{baykal2018data} propose to sample the parameters based on their importance, quantified by the sensitivity of the output due to parameter changes, and approximated using the training data. 
    \citet{mussay2019data} use the same idea of constructing coresets, but calculate the sensitivity in a data-free manner. \citet{ye2020good} propose to reconstruct a network by greedily adding the neurons that decrease the prediction risk the most. They show the generalization error of the constructed model with $m$ parameters is at the order of $O(1/m)$ for two-layer neural networks. However, the prediction risk is intractable in practice, and the greedy selection is done by evaluating the training loss, causing a gap from the theory. 
    \citet{malach2020proving, orseau2020logarithmic} prove that any network can be arbitrarily well approximated by a subnetwork pruned from another deeper and wider network.  \citet{zhang2021lottery} show why the pruned network often has a better test accuracy and easier to train than the original network. 
    
    In this paper, we aim to study two fundamental but under-explored problems: for a pre-trained neural network,
    when can one perform a successful pruning, and how much can one compress with little accuracy degradation. 
    We are motivated by the idea that an essentially sparse network has fewer important weights and can be pruned more without affecting the performance. An immediate example of sparsity indicator is the number of non-zero weights in a network, also known as the $\ell_0$-norm or hard sparsity. However, it is common in practice that the network has many weights with small magnitudes. We propose to use the $\ell_q$-norm ($0<q<1$) of weights as a soft sparsity measure to characterize the compressibility of a network.
    In particular, we will provide an upper bound for the approximation and generalization errors of the pruned model for any given pruning ratio. This bound indicates that a network with larger soft sparsity is more compressible. It also guides the selection of a proper pruning ratio that produces allowed accuracy degradation. 
    We then propose a novel adaptive backward pruning procedure based on our theory, which has two specific implementations. 
    The first implementation prunes the network based on the magnitude of weights, while the pruning ratio of each neuron is adaptively determined by its soft sparsity level.
    The second implementation alternatively determines the pruning ratio of each neuron by using LASSO~\citep{tibshirani1996regression} with the same penalty parameter. A neuron that is essentially sparse will be pruned more by this way. Experiments show their promising performance when compared with the baseline method that prunes a fixed proportion of weights with the smallest magnitude neuron-wisely.

\section{Problem formulation}\label{sec:form}
    We start with the standard regression learning framework. We will extend the results to classification learning in \Autoref{sec:thm}. Suppose that the data generation distribution is $(X,Y) \sim \mu$, where $X = (x_1, \ldots, x_p)^\T \in \mathbb{R}^p$ denotes the predictor/input variable, $Y \in \real$ is the response/output that is square-integrable with respect to $\mu$. The target function is $f^*: x \to \E(Y \mid X=x)$, and the loss function is the mean squared error loss throughout the paper unless otherwise specified.
    The risk, or generalization error, of a function $f: \real^p \to \real$ is given by $R(f) = \E{[f(X)-f^*(X)]^2}.$

    A pre-trained $L$-layer fully connected neural network, denoted by $f_T$, is given as follows. The initial input layer is considered as the $0$-th layer, and the output layer is the $L$-th layer. 
    \begin{align*}
        f_i^{(0)} &= x_i, \ i=1,\ldots, p,
        \\
        g_i^{(k)} &= \sum^{n_{k-1}}_{j=1}w_{ij}^{(k-1)}f_j^{(k-1)}
        , \ f_i^{(k)} = \sigma\bigl(g_i^{(k)}\bigr), i=1,\ldots,n_k, k=1,\ldots, L,
        \\ f_T &= g_1^{(L)}.
    \end{align*}
    Here, we call $f_i^{(k)}$ the $i$-th function, or neuron, in the $k$-th layer, $g_i^{(k)}$ the linear part of $f_i^{(k)}$,
    $n_k$ is the number of neurons in the $k$-th layer ($n_0=p$, $n_L=1$), $w_{ij}^{(k-1)}$'s are the weights or parameters\footnote{Without loss of generality, the bias term is absorbed into the weights since we can add a constant neuron to each layer.},
    and $\sigma(\cdot)$ is a $\rho$-Lipschitz activation function. Common activation functions include the ReLU function $\sigma(x) = \max\{x, 0\}$, tanh function $\sigma(x)=(e^{x}+e^{-x})/(e^{x}-e^{-x})$, and sigmoid function $\sigma(x)=1/(1+e^{-x})$. 
    

    We are interested in finding a sparsified network $f_S$ of $f_T$, such that $f_S$ has a similar generalization error as $f_T$. In other words, $f_S$ has the same structure as $f_T$, but the weights include many zeros. The non-zero elements of $f_S$ do not have to remain the same as those in $f_T$. Suppose that a vector $w$ has $M$ coefficients in total, and after pruning, it has only $m$ non-zero coefficients. We define the compression ratio and pruning ratio for this vector as $M/m$ and $1-m/M$, respectively. This definition naturally generalizes to a layer or the whole network. Thus, the problem is cast into the following: what is the best generalization error of a subnetwork $f_S$ with a given pruning ratio?
    
    \textbf{Notation.} We use $\E$ for expectation and $\P$ for probability.
    For a vector $w=(w_1, \ldots, w_d)^\T \in \real^d$, the $\ell_0$-norm and $\ell_q$-norm with $0< q \leq 1$ of $w$ is 
    \begin{align*}
        \norm{w}_0=\sum_{i=1}^d \ind_{w_i \neq 0},\ \text{ and } \norm{w}_q=\biggl[\sum_{i=1}^d \abs{w_i}^q\biggr]^{1/q}, 
    \end{align*}
    where $\ind_{(\cdot)}$ is the indicator function. 
    Note that the $\ell_q$-norm with $0<q<1$ is actually a quasinorm, and $\ell_0$-norm is not even a quasinorm, though we still call them norms by convention. 
    Define the $L_p(\mu_X)$-norm of any function $f$ for $p \geq 1$ as
    \begin{align*}
        \norm{f}_p = \biggl(\int |f|^p\mu_X(dx)\biggr)^{1/p} ,
    \end{align*}
    where $\mu_X$ is the marginal distribution of $X$. The generalization error of $f$ can be written as $R(f) = \norm{f-f^*}_2^2.$ The $L_{\infty}(\mu_X)$-norm of $f$ is $\norm{f}_{\infty} = \esssup \abs{f}.$ Other frequently used notations are summarized in \Autoref{tab:notation}. 
    
    \begin{table}
     \caption{A summary of frequently used notations.}
        \label{tab:notation}
        \centering
        \begin{tabular}{lp{10.5cm}}
        \toprule
        Notation & Meaning \\
        \midrule
            $f_j^{(k)}, g_j^{(k)}$ & The $j$-th function (also named neuron) and its linear part (before activation) in the $k$-th layer \\
            $f_j^{(k),s}, g_j^{(k),s}$& The approximation of $f_j^{(k)}, g_j^{(k)}$ after $s$ steps\\
            $f^*, f_T, f_{T}^{(s)}$ & The target function, the pre-trained network, and the approximation of $f_{T}$ after $s$ steps \\
            $w_{ij}^{(k-1)}$ or $w_{i,j}^{(k-1)}$ & The $i$-th coefficient of the $j$-th function's weight vector in the $k$-th layer \\
            $n_k, m_k, q_k, t_k$ & The number of neurons, number of neurons preserved, the norm, and the largest $\ell_{q_k}$-norm of weights in the $k$-th layer \\
            $\sigma(\cdot)$ & A $\rho$-Lipschitz activation function \\
        \bottomrule
        \end{tabular}
    \end{table}
    
\section{Theoretical characterization of compressibility using \texorpdfstring{$\ell_q$}{} norm}\label{sec:thm}
    We next provide an upper bound of the generalization error for some pruned model $f_S$. 
    
    Our study is motivated by sparse linear approximation. For any function in $f=\sum_{i=1}^n w_if_i$ that is square-integrable, we hope to approximate $f$ by a linear combination of a small subset of basis functions $\mathcal{F}=\{f_1, \ldots, f_n\}$. Specifically, we would like to choose $m$ functions $\mathcal{F}_m=\{f_{s_1}, \ldots, f_{s_m}\}$ such that
    $\hat{f}=\sum_{i=1}^m \tilde{w}_{s_i}f_{s_i}$ is close to $f$, where $\tilde{w}_{s_i}$'s are new coefficients. We note that this is model compression of the one layer neural network scenario when $f$ is considered as the output layer and $f_i$'s are functions of neurons in the hidden layer.

    It is intuitive that if the weight vector $w$ of $f$ is sparse, then $f$ can be well approximated by a small subset of $f_i$'s. For example, suppose $w$ has the hard sparsity of $\norm{w}_0=m$, we can simply retain $f_i$ with $w_i\neq 0$. However, hard sparsity is unrealistic.
    In practice, a much more common situation is that the network has a large number of small-valued coefficients.
    A notable approach to study sparse linear approximation in the latter case is via the $\ell_q$-norm ($0<q<1$) of $w$~\citep{wang2014adaptive}. For a vector $w$ with fixed dimension, a small $\norm{w}_q$ implies that the number of relatively large parameters is also limited. Therefore, $\ell_q$-norm can be regarded as the soft sparsity indicator.
    In particular, for a given subset cardinality $m$, \citep{wang2014adaptive} provides an upper bound of the approximation error for the best linear combination of $m$ basis functions under some mild conditions. 
    
    This paper utilizes the above result and extends to any fully connected neural network. Specifically, we approximate each neuron by a small number of neurons from the previous layer, hence obtaining a simpler model, and establish its error bound. In other words, each neuron in the previous layer is regarded as a basis function for current layer's neurons, and we compress the connections between them by using a sparse linear combination. The overall compressibility of a network is evaluated by the aggregation of compressibility for all neurons.
    
    Next, we describe how to apply the sparse linear approximation to the whole network in detail. We consider a backward approximation scheme. Let $S$ index the step of approximations starting from the output layer. The first-step approximation step, $S=1$, is to select $m_{L-1}$ neurons from the $(L-1)$-th layer to obtain a linear combination of $f_j^{(L-1)}, j=1,\ldots,n_{L-1}$ as the approximation for $f_T$, which is denoted by $f_{T}^{(1)}$.
    Without loss of generality, we assume the first $m_{L-1}$ neurons are selected, since we can always reorder the indices. Thus,
    \begin{align}\label{eq:step1}
        f_{T}^{(1)} = \sum_{i=1}^{m_{L-1}} \tilde{w}_{1,i}^{(L-1)} f_i^{(L-1)}, 
    \end{align}
    where $\tilde{w}_{1,i}^{(L-1)}$ is the new weights for the sparse approximation. 
    The second-step approximation of $S=2$ is to approximate each of the $f_j^{(L-1)}, j=1,\ldots,n_{L-1}$, by selecting $m_{L-2}$ neurons from the $(L-2)$-th layer. 
    Here, we assume the pruning ratio for the neurons in the same layer is fixed for simplicity.
    For each $f_j^{(L-1)}$, suppose the indices of functions selected to approximate it are $\{s_1, \ldots, s_{m_{L-2}}\}$, we have the one-step approximation of $f_j^{(L-1)}$ as
    \begin{align}\label{eq:step2}
        f_j^{(L-1),1} = \sigma\bigl(g_j^{(L-1),1}\bigr), \  g_j^{(L-1),1} = \sum_{i=1}^{m_{L-2}} \tilde{w}_{j,s_i}^{(L-2)} f_{s_i}^{(L-2)}.
    \end{align}
    Note that the neurons selected for approximating different $f_{j}^{(L-1)}$ are different. To ease the notation, we do not distinguish the index sets for different neurons and use $\{s_1, \ldots ,s_{m_k}\}$, which should be self-explanatory in the context.
    Plugging the approximation above into $f_{T}^{(1)}$, we obtain the two-step approximation of $f_T$ as follows.
    \begin{align*}
        f_{T}^{(2)} = \sum_{i=1}^{m_{L-1}} \tilde{w}_{1,i}^{(L-1)} f_i^{(L-1),1}. 
    \end{align*}
    Iteratively, after the $S$-step approximation, the output function is approximated $S$ times, denoted by $f_{T}^{(S)}$.
    
    In summary, each neuron in the $(k+1)$-th layer of the pre-trained network $f_T$ has $n_k$ inputs from the previous layer, and we prune the number of inputs to $m_k$ (and tune the weights correspondingly). For all such sub-networks, we give an upper bound of the generalization error for the best sub-network as follows.
    \begin{theorem}[Error bound of the pruned model]\label{thm4.1}
    Let $f$ be any square-integrable function and $t_k=\max_{1\leq i \leq n_k}\norm{w_i^{(k)}}_{q_k}$ for $k=0, \ldots, L-1$, where $w_i^{(k)}= (w_{i,1}^{(k)}, \ldots, w_{i,n_{k}}^{(k)})$ is the weight vector of $i$-th function of $(k+1)$-th layer.
    For any $1\leq S\leq L$, $1\leq m_k \leq n_k$, and $0<q_k\leq 1$, $k=0, \ldots, L-1$, we have 
    \begin{align} \label{eq32}
        &\bnorm{f - f_{T}^{(S)}}_2 \leq
    \norm{f - f_T}_2 + Ct_{L-1}(m_{L-1})^{1/2-1/q_{L-1}}\max_{1 \leq j \leq n_{L-1}}\big\|f_j^{(L-1)}\big\|_2 \\ \nonumber
        &+ \rho C t_{L-1}t_{L-2}(m_{L-2})^{1/2-1/q_{L-2}}\max_{1 \leq j \leq n_{L-2}}\big\|f_j^{(L-2)}\big\|_2 + \ldots \\ \nonumber
        &+ \rho^{S-1}Ct_{L-1}t_{L-2}\ldots t_{L-S}(m_{L-S})^{1/2-1/q_{L-S}}\max_{1 \leq j \leq n_{L-S}}\big\|f_j^{(L-S)}\big\|_2,
    \end{align}
    where $C$ is a universal constant.
    \end{theorem}
    The proofs of \Autoref{thm4.1} and subsequent results are included in the Appendixs.
    \begin{remark}
    Note that if we take $f=f^*$, then the left hand side of \ref{eq32} is exactly the generalization error of the pruned model $f_{T}^{(S)}$, and is upper bounded by the generalization error of the original network $\norm{f^* - f_T}_2$ plus the approximation error.
    \end{remark}
    \begin{remark}
    \Autoref{thm4.1} indicates that there exists a sub-network $f_T^{(S)}$ such that the accuracy degradation is properly upper bounded. The bound~\ref{eq32} further implies that there is a trade-off between the generalization error and the compression ratio. Namely, the smaller $m_k$, the higher compression ratio, and the larger upper bound of the generalization error.
    \end{remark}
    \begin{remark}
    The compressibility, or the generalization error bound, is characterized by the $\ell_q$-norm of the original network. 
    As mentioned at the beginning of this section, a small $\ell_q$-norm ($t_k$) indicates that the number of important weights are limited and the network is sparse, and thus one can compress more. 
    This characterization can be used to understand the two fundamental issues proposed in the introduction. In particular, for the first question of when one can prune a network, we know a network with smaller soft sparsity can be pruned more. For the second question that how much one can prune with controlled accuracy drop, we can derive a lower bound for the pruning ratio via~\ref{eq32}.
    \end{remark}
    \begin{remark}
    \Autoref{thm4.1} assumes the pruning ratio and $\ell_q$-norm are fixed for all neurons in the same layer for technical convenience. In practice, one may allow them to differ for each neuron. To illustrate this point, we will propose adaptive neuron-level pruning techniques in \Autoref{ch4.3.3}. Additionally, \Autoref{thm4.1} can be generalized to give an upper bound for pruning certain layers, not necessarily in this backward form. 
    \end{remark}
    \begin{remark}
    The condition that $t_k=\max_{1\leq i \leq n_k}\norm{w_i^{(k)}}_{q_k}$ can be relaxed to $t_k=\max_{i \in \mathcal{I}_k}\norm{w_i^{(k)}}_{q_k}$, where $\mathcal{I}_k$ contains the index of neurons that connect to the $(k+1)$-th layer of the pruned model. 
    \end{remark}
    
    Next, we present a more parsimonious form of \Autoref{thm4.1} with specific activation functions and pruning ratios.
    \begin{corollary}[Homogeneous pruning]\label{coro:prune}
    For activation function such as tanh and sigmoid, we have $\rho=1$ 
    and $\norm{f^{(i)}_j}_2 \leq 1$. Then, the bound~\ref{eq32} can be replaced with a simpler one
    \begin{align*}
        \big\|f - f_{T}^{(S)}\big\|_2 
        &\leq  \big\|f - f_T\big\|_2 + Ct_{L-1}(m_{L-1})^{1/2-1/q_{L-1}}  
        +\cdots 
        \\ &
         \quad + Ct_{L-1}t_{L-2}\ldots t_{L-S}(m_{L-S})^{1/2-1/q_{L-S}}.
    \end{align*}
    Furthermore, when $q_k = q$ and $m_k = m$ for all $k$, we have
    \begin{align} \label{eq34}
        \big\|f - f_{T}^{(S)}\big\|_2 \leq  \big\|f - f_T\big\|_2 + C(t_{L-1} + \ldots + t_{L-1}t_{L-2}\ldots t_{L-S})m^{\frac{1}{2}-\frac{1}{q}}.
    \end{align}    
    \end{corollary}
    
    As an application of Corollary~\ref{coro:prune}, we provide an example to illustrate how one may decide the pruning ratio to achieve a desired generalization error rate.
    \begin{example}
    Suppose that $f_T$ is trained with data of sample size $N$, and take $ n_{L-1} = \ldots = n_{1} = N^\alpha$ for some $\alpha >0$. Suppose that $t_{L-1} = \cdots = t_1 = t = O(N^\gamma)$ for some small $\gamma >0$, where $O$ is the standard big $O$ notation. Then, by choosing $f=f^*$ in \Autoref{eq34}, we have
    \begin{align*}
        \big\|f^* - f_{T}^{(S)}\big\|_2 
        &\leq  \big\|f^* - f_T\big\|_2 + Cm^{\frac{1}{2}-\frac{1}{q}}(t + t^2+\ldots + t^S) \\
        &= \big\|f^* - f_T\big\|_2 + O(N^{S\gamma})m^{\frac{1}{2}-\frac{1}{q}}.
    \end{align*}
    
    For any $0<\tau \leq \alpha(1/q-1/2)-S\gamma$, choosing $m = N^{(S\gamma + \tau)2q/(2-q)}$ yields
    \begin{align*}
        \big\|f^* - f_{T}^{(S)}\big\|_2 &\leq  \big\|f^* - f_T\big\|_2 + O(N^{-\tau}),
    \end{align*}
    which guarantees an error bound of $O(N^{-\tau})$. The associated compression rate is $N^\alpha/m$. 
    \end{example}
    
    \textbf{Extension to classification learning.} 
    We consider a binary classification task for simplicity. Suppose that $Y \in \{-1, 1\}$. For any function $f: \real^d \to [0, 1]$, let the classification rule be $\delta_f: x\to \sign(2f(x)-1)$, where $\sign(x) = 1$ if $x >0$ otherwise $\sign(x)=-1$. Suppose we use the zero-one loss function, so the classification error probability of $f$ is given by
    \begin{align*}
        R(f) = \P(\delta_f(X)\neq Y). 
    \end{align*}
    Note that $f^*:x \to \E(Y\mid X=x)$ is still a minimizer of the error. 
    The following well-known inequality connects the classification error of $\delta_f$ and the regression error of $f$~\citep{devroye2013probabilistic}.
    \begin{lemma}
    For any function $f: \real^d \to [0, 1]$, we have 
    \begin{align*}
        R(f) - R(f^*) \leq 2\norm{f-f^*}_1 \leq 2\norm{f-f^*}_2.
    \end{align*}
    \end{lemma}
    Combining the above result with \Autoref{thm4.1}, we immediately have the following result.
    \begin{corollary}[Classification]
    For a binary classification task described above, we have  
    \begin{align*}
        R(f_T^{(S)}) - R(f^*) &\leq 2\norm{f^* - f_T}_2 + Ct_{L-1}(m_{L-1})^{1/2-1/q_{L-1}}\max_{1 \leq j \leq n_{L-1}}\big\|f_j^{(L-1)}\big\|_2 \\ 
        &\quad+ C\rho t_{L-1}t_{L-2}(m_{L-2})^{1/2-1/q_{L-2}}\max_{1 \leq j \leq n_{L-2}}\big\|f_j^{(L-2)}\big\|_2 + \ldots \\ \nonumber
        &\quad+ C\rho^{S-1}t_{L-1}t_{L-2}\ldots t_{L-S}(m_{L-S})^{1/2-1/q_{L-S}}\max_{1 \leq j \leq n_{L-S}}\big\|f_j^{(L-S)}\big\|_2.
    \end{align*}
    \end{corollary}

\section{Adaptive pruning algorithms}\label{ch4.3.3}

\subsection{Overview of the pruning procedure}

    Based on the developed theory in \Autoref{sec:thm}, we propose an adaptive backward pruning procedure (`ABP'). Specifically, we start from the last layer, which is the $L$-th layer of one function, and approximate it by constructing a sparse linear combination of the functions in the $(L-1)$-th layer, as presented in \Autoref{eq:step1}. We then proceed to the second last layer, and apply a similar procedure to the linear part of each neuron there, as presented in \Autoref{eq:step2}. We repeat approximating neurons from back to the front until reaching the first layer of the network. Additionally, we choose the pruning ratio of each neuron in an adaptive manner according to its soft sparsity level. In particular, a neuron has larger soft sparsity will be pruned more. We summarize the overall pruning procedure in \Autoref{alg:m}.
    
    Since \Autoref{thm4.1} only shows the existence of a pruned model that satisfies the error bound, we need a practical algorithm to find the sparse linear approximation for each neuron. To this end, we propose two particular adaptive pruning strategies, also summarized in \Autoref{alg:mag} and \Autoref{alg:lasso} as subroutines of \Autoref{alg:m}, respectively.
    
\subsection{Adaptive pruning strategies for each neuron}

    \textbf{Magnitude-based pruning (\aone).} This algorithm to find sparse linear approximation is based on the magnitude of the weights. Let $I_m$ be the largest $m$ components (in absolute values) of a vector $w$. Since the energy of $w$ is concentrated in the coefficients indexed by $I_m$, a natural idea is to prune all the coefficients not in $I_m$. The question is how to decide the pruning ratio, or equivalently, $m$. 
    To address that, we first introduce a tolerance parameter $\eta$ that satisfies $\sum_{i \notin I_m}|w_i|^q\leq \eta \sum_{i \in I_m}|w_i|^q$. We use $\eta$ to control the overall pruning degree, since a smaller $\eta$ requires a larger $m$. For any fixed $\eta$, 
    we propose to decide $m$ by each neuron's soft sparsity. In particular, we
    define the sparsity index for $w \in \real^d$ and any $0<q<1$ as
    \begin{align*}
        \SI_q(w) = \norm{w}_1/\norm{w}_q.
    \end{align*}
    We will write $\SI_q(w)$ as $\SI$ for short in the rest of the paper when there is no ambiguity.
    We can derive the following inequality (with more details in the Appendix).
    \begin{align}
        m \geq \SI^{-q/(1-q)}(1+\eta)^{-1/(1-q)}. \label{eq4.si}
    \end{align}
    We propose to choose $m$ as the lower bound in \ref{eq4.si}. The corresponding algorithm is summarized in \Autoref{alg:mag}. 
    
    \begin{remark}
    Note that $\SI \in [d^{1-1/q}, 1]$, and a larger $\SI$ indicates a sparser vector. Aligned with our motivation, \Autoref{eq4.si} indicates that if $w$ is soft-sparse and $\SI$ is relatively large, then $m$ can be small and we prune more.
    \end{remark}
    Similar to \Autoref{thm4.1}, we can show an upper bound for the pruned model produced by \Autoref{alg:mag} as follows.
    \begin{theorem}\label{thm:mag}
    Suppose $f_{T}^{(S)}$ is obtained by \Autoref{alg:mag}, with $m_{k}$ defined as the smallest $m$ obtained for the $k$th layer.
    The other notation and conditions are the same as \Autoref{thm4.1}. Then, we have
    \begin{align*} 
        \bnorm{f - f_{T}^{(S)}}_2 \leq &
    \norm{f - f_T}_2 + t_{L-1}(m_{L-1})^{1-1/q_{L-1}}\max_{1 \leq j \leq n_{L-1}}\big\|f_j^{(L-1)}\big\|_2 \\ \nonumber
        &+ \rho t_{L-1}t_{L-2}(m_{L-2})^{1-1/q_{L-2}}\max_{1 \leq j \leq n_{L-2}}\big\|f_j^{(L-2)}\big\|_2 + \cdots \\ \nonumber
        &+ \rho^{S-1}t_{L-1}t_{L-2}\ldots t_{L-S}(m_{L-S})^{1-1/q_{L-S}}\max_{1 \leq j \leq n_{L-S}}\big\|f_j^{(L-S)}\big\|_2.
    \end{align*}
    \end{theorem}
    \begin{remark}
    We note that the upper bound of \Autoref{thm:mag} is looser than \Autoref{thm4.1}, since pruning based on magnitude is a simple but possibly crude choice of the sparse linear combination, while \Autoref{thm4.1} prunes the network based on the best sparse approximation. This also motivates us to use other strategies to find a better sparse approximation, such as the application of LASSO discussed next.
    \end{remark}

    \begin{algorithm}[tb]
    \caption{One-shot adaptive backward pruning procedure (`ABP')}\label{alg:m}
    \begin{algorithmic}[1]
        \Require Pre-trained network to be pruned $f_T$ 
        \State  $k=L$ 
        \While{$k>0$}
            \For{$j$ in $1,\dots,n_k$}
                \State Get the weights vector $w_j^{(k-1)}$ for $j$-th neuron in $k$-th layer of $f_T$
                \State Adaptively find its sparse linear approximation $\tilde{w}_j^{(k-1)}$, such as using \Autoref{alg:mag} or \ref{alg:lasso}
            \EndFor
            \State $ k = k-1$
        \EndWhile
        \Ensure The pruned model $f_S$ with weights $\tilde{w}_j^{(k-1)}, j=1,\ldots,n_k, k=1,\ldots,L$.
    \end{algorithmic}
    \end{algorithm}
    
    \begin{algorithm}[tb]
    \caption{(Subroutine of \Autoref{alg:m}, \aone) Sparse approximation based on magnitude}\label{alg:mag}
    \begin{algorithmic}[1]
        \Require Weight vector $w_j^{(k-1)}$, $\eta$, $q$ \Comment{Approximate the $j$-th neuron in $k$-th layer}
        \State Calculate $\SI = \bnorm{w_j^{(k-1)}}_1/\bnorm{w_j^{(k-1)}}_q$
        \State Calculate $m = \SI^{q/(q-1)}(1+\eta)^{1/(q-1)}$
        \State Select $m$ indices of $w_j^{(k-1)}$ that have the largest magnitudes as $J_m=\{s_1, \dots, s_m\}$
        \State Calculate $X^{(k)}_i = (f_{s_1}^{(k-1)}(X_i), \ldots, f_{s_m}^{(k-1)}(X_i))$, $Y^{(k)}_i = g_j^{(k)}(X_i)$ for $i=1,\ldots, N$
        \State Perform linear regression on $X^{(k)}_i, Y^{(k)}_i, i=1,\ldots, N$. 
        \Ensure The retrained weights $\tilde{w}$ from the linear regression
    \end{algorithmic}
    \end{algorithm}
    
    \begin{algorithm}[tb]
    \caption{(Subroutine of \Autoref{alg:m}, \atwo) Sparse approximation using LASSO}\label{alg:lasso}
    \begin{algorithmic}[1]
        \Require Weight vector $w_j^{(k-1)}$, $\lambda$ \Comment{Approximate the $j$-th neuron in $k$-th layer}
        \State Calculate $X^{(k)}_i = (f_1^{(k-1)}(X_i), \ldots, f_{n_{k-1}}^{(k-1)}(X_i))$, $Y^{(k)}_i = g_j^{(k)}(X_i)$ for $i=1,\ldots, N$
        \State Perform LASSO with $X^{(k)}_i, Y^{(k)}_i, i=1,\ldots, N$ and penalty parameter $\lambda$
        \Ensure The retrained weights $\tilde{w}$ from LASSO
    \end{algorithmic}
    \end{algorithm}
    
    \textbf{LASSO-based pruning (\atwo).} As an alternative, we may find sparse linear approximations using LASSO~\citep{tibshirani1996regression}. To approximate the $j$-th neuron in the $k$-th layer using the functions in the $(k-1)$-th layer, we first obtain the input and output of this neuron using the training data. They are denoted by $X^{(k)}_i = (f_1^{(k-1)}(X_i), \ldots, f_{n_{k-1}}^{(k-1)}(X_i))$ and $Y^{(k)}_i = g_j^{(k)}(X_i)$ for $i=1,\ldots, N$, respectively. The approximation weight vector $\tilde{w}$ is obtained from applying LASSO to $X^{(k)}_i, Y^{(k)}_i, i=1,\ldots, N$ with penalty parameter $\lambda$. We note that LASSO adds a $\ell_1$ penalty on the weight vector, which enforces the learned $\tilde{w}$ to be sparse. Furthermore, an essentially sparser nature (larger soft sparsity) leads to a sparser $\tilde{w}$. 
    This algorithm is summarized in \Autoref{alg:lasso}.

\section{Experiments} \label{sec:exp}
    We compare our proposed pruning procedure using magnitude-based approximation \Autoref{alg:mag} (\aone) and LASSO-based \Autoref{alg:lasso} (\atwo) with a standard pruning algorithm (\base), which prunes a fixed proportion $p$ of weights with the smallest magnitude for each neuron. We train a four-layer fully connected ReLU neural network $f_T$ on the California Housing dataset~\citep{pace1997sparse}, which is a regression task with eight continuous predictors and about 20 thousand instances. The evaluation criterion is mean squared error (MSE). For \aone, we choose hyper-parameters $q=0.3,0.5,0.7$ and $\eta=0, 0.1, 0.2, 0.3$. For \atwo, we use penalty parameter $\lambda=10^{-5}, 10^{-4}, 10^{-3}$. For \base, we choose pruning ratio $p=0.3, 0.5, 0.7$. Each time, we prune $f_T$ using all methods under different settings, and then evaluate the compression ratio, pruning ratio (both defined in \Autoref{sec:form}), and the MSE increase ratio,
    which is the increase of MSE (the MSE difference between the pruned and the original networks) divided by the MSE of $f_T$. The procedure is replicated $20$ times, and the results are summarized in \Autoref{fig:pr_ac}. 
    
    \begin{table}[tb]
        \centering
                \caption{Mean compression ratio, pruning ratio, and MSE increase ratio for three methods with different hyper-parameters. The stand errors are reported in the parentheses.}
        \label{fig:pr_ac}
        \begin{tabular}{lccc}
        \toprule
             Method	& Compression Ratio	&Pruning Ratio	& MSE Increase Ratio\\
             	\midrule
ABP-L ($\lambda=10^{-3}$) & 12.97 (0.22) & 0.92 (0.00) & 0.30 (0.02) \\
 ABP-L ($\lambda=10^{-4}$)&  4.68 (0.07) & 0.79 (0.00) & 0.05 (0.00) \\
 ABP-L ($\lambda=10^{-5}$)&  2.33 (0.04) & 0.57 (0.01) & 0.03 (0.00) \\ \midrule
 ABP-M ($\eta=0$, $q=0.3$)&  1.52 (0.01) & 0.34 (0.00) & 0.01 (0.00) \\
 ABP-M ($\eta=0$, $q=0.5$)&  1.80 (0.02) & 0.44 (0.01) & 0.01 (0.00) \\
 ABP-M ($\eta=0$, $q=0.7$)&  2.10 (0.02) & 0.52 (0.01) & 0.02 (0.00) \\
 ABP-M ($\eta=0.1$, $q=0.5$)&  2.46 (0.03) & 0.59 (0.01) & 0.04 (0.01) \\
 ABP-M ($\eta=0.2$, $q=0.5$)&  3.12 (0.06) & 0.68 (0.01) & 0.17 (0.06) \\
 ABP-M ($\eta=0.3$, $q=0.5$)&  3.88 (0.08) & 0.74 (0.01) & 0.31 (0.09) \\ \midrule
   Mag ($p=0.3$)&  1.55 (0.00) & 0.35 (0.00) & 0.07 (0.01) \\
   Mag ($p=0.5$)&  2.32 (0.01) & 0.57 (0.00) & 0.41 (0.06) \\
   Mag ($p=0.7$)&  4.71 (0.10) & 0.79 (0.00) & 0.85 (0.10) \\
\bottomrule
        \end{tabular}
    \end{table}
    
    From \Autoref{fig:pr_ac}, both \aone~and \atwo~have a significantly smaller accuracy degradation ratio compared to \base, when the compression ratio is similar, which supports our intuition that an adaptive pruning scheme is more efficient than pruning a fixed portion of weights. Furthermore, \atwo~outperforms \aone~in general. 
    
    For \aone, we note that the sparsity index-inspired pruning in~\Autoref{eq4.si} with $\eta=0$ (or small) works very well for preserving accuracy, although it may be conservative in terms of the pruning ratio.
    Regarding the choice of $q$, a large $q$ tends to increase the pruning ratio, but as long as $\eta$ is chosen appropriately, $q$ is relatively insensitive. As for \atwo, a larger penalty parameter $\lambda$ leads to a larger pruning ratio.
    Since both methods are one-time pruning, we suggest selecting hyper-parameters through cross-validation to balance deep pruning and accuracy protection. The code is included in the Appendix.

\section{Conclusion}\label{sec:con}
    This paper provides a theory that characterizes the compressibility of a neural network in terms of the $\ell_q$-norm of its weights. The $\ell_q$-norm, or soft sparsity, can be used to compare the compressibility of different models with the same structure. Furthermore, it reveals the relationship between the degree of compression and accuracy degradation, which guides us in selecting an appropriate pruning ratio. The theory also motivates a new pruning scheme by finding a sparse linear approximation of neurons in a backward manner. The developed algorithms produce pruned models with significantly better performance than some standard pruning algorithms.
    
    There are some limitations of the current study that we leave for future work. First, our pruning procedure is one-shot pruning, so we may rely on cross-validation to select the hyper-parameters for optimal performance. It will be interesting to study the stopping criterion and develop an iteratively pruning algorithm that can stop intelligently with maximal pruning ratio and little accuracy drop. Second, how to fairly compare the compressibility between two networks with different structures remains a challenge. Third, we focused exclusively on fully connected feed forward neural networks. Generalizations of our results to other networks are of interest.


\section*{Appendices}
\appendix

\section{Proof of Theorem 1}
    We first introduce the following lemma.

    \begin{lemma}[\citet{gao2013metric, wang2014adaptive}]\label{lem4.1}
    Let $f$ be any function with $\big\| f \big\|_2 < \infty$. Suppose $\mathcal{F} = \{f_1, \ldots, f_M\}$ with $\max_{1 \leq j \leq M}\big\| f_j \big\|_2 < \infty$. For any $0 < q \leq 1$ and $t>0$, define the $l_{q,t}$-hull of $\mathcal{F}$ to be the class of linear combinations of functions in $\mathcal{F}$ with the $\ell_q$-constraint
    \begin{align*}
        \mathcal{F}_q(t) = \biggl\{f_\theta = \sum^m_{j=1}\theta_j f_j : \big\|\theta \big\|_q \leq t, f_j \in  \mathcal{F} \biggr\}.
    \end{align*}
    For any $1 \leq m \leq M$ and $t>0$, there exist a subset $J_m=\{s_1, \ldots, s_m\}$ of $\{1, \ldots, M\}$ and 
    \begin{align*}
        f_{\theta^m} = \sum^m_{j=1}\theta^m_j f_{s_j} \in \mathcal{F}_{J_m} = \textrm{span of } f_j (j \in J_m) 
    \end{align*}
    with $\big\|\theta^m \big\|_1 \leq t$ such that
    \begin{align} \label{eq4.3}
        \big\|f - f_{\theta^m}\big\|_2 \leq \big\|f - f_{\theta^*}\big\|_2 + tCm^{\frac{1}{2}-\frac{1}{q}}\max_{1 \leq j \leq M}\big\|f_j\big\|_2,
    \end{align}
    where $f_{\theta^*} = \arg\min_{f_\theta \in \mathcal{F}_{q(t)}}\big\|f - f_{\theta}\big\|_2$ and $C$ is an absolute constant.
    \end{lemma}
    
    The sketch proof idea of \Autoref{thm4.1} is applying \Autoref{lem4.1} from the last layer of $f_T$ to the first layer step-wisely. 
    We use the induction for the proof.
    
    When $S = 1$, we try to approximate the last layer $f_T$, or $g_1^{(L)}$, using the linear combination of functions  $\mathcal{F}^{(L-1)} = \{f_1^{(L-1)}, \ldots, f_{n_{L-1}}^{(L-1)}\}$ in the $(L-1)$-th layer.
    Invoke \Autoref{lem4.1} with $m=m_{L-1}$, $q=q_{L-1}$ and $t=t_{L-1}=\max_{1\leq i\leq n_{L-1}}\norm{w_i^{(L-1)}}_{q_k}$, there exists $J_{m_{L-1}} = \{s_1, \ldots, s_{m_{L-1}}\} \subset \{1, \ldots, n_{L-1}\}$ and $f_{T}^{(1)}=\sum^{m_{L-1}}_{j=1}\Tilde{w}_{1,s_j}^{(L-1)} f_{s_j}^{(L-1)}$ such that 
    \begin{align}\label{eq:s1}
        \big\|f - f_{T}^{(1)}\big\|_2 &\leq \big\|f - f_{\theta^*}\big\|_2 + Ct_{L-1}(m_{L-1})^{\frac{1}{2}-\frac{1}{q_{L-1}}}\max_{1 \leq j \leq n_{L-1}}\big\|f_j^{(L-1)}\big\|_2, \nonumber \\
        &\leq \big\|f - f_T\big\|_2 + Ct_{L-1}(m_{L-1})^{\frac{1}{2}-\frac{1}{q_{L-1}}}\max_{1 \leq j \leq n_{L-1}}\big\|f_j^{(L-1)}\big\|_2. 
    \end{align}
    where 
    \begin{align*}
        f_{\theta^*} = \argmin_{f_{\theta} \in \mathcal{F}^{(L-1)}_{q_{L-1}}(t_{L-1})}\big\|f - f_{\theta}\big\|_2
    \end{align*}
    and 
    \begin{align}\label{eq:w1}
        \big\|(\Tilde{w}_{1,s_1}^{(L-1)}, \ldots, \Tilde{w}_{1,s_{m_{L-1}}}^{(L-1)})\big\|_1 \leq t_{L-1}.
    \end{align}
    The last inequality of~\ref{eq:s1} holds since $f_T \in \mathcal{F}^{(L-1)}_{q_{L-1}}(t_{L-1})$. 
    
    For $S = 2$, we are going to approximate each $f_j^{(L-1)}$ by functions in the $(L-2)$-th layer.
    In particular, we invoke \Autoref{lem4.1} to $g_j^{(L-1)}$, the linear part of $f_j^{(L-1)}$, with $m=m_{L-2}$, $q=q_{L-2}$, $t=t_{L-2}=\max_{1\leq i\leq n_{L-2}}\norm{w_i^{(L-2)}}_{q_k}$ and $\mathcal{F}^{(L-2)} = \{f_j^{(L-2)}, 1\leq j \leq n_{L-2}\}$.
    For any $1\leq j \leq m_{L-1}$, there exists a subset $J_{m_{L-2}}=\{s_1, \ldots, s_{m_{L-2}}\}$ (note that this subset varies for different $j$) and a sparse linear approximation 
    \begin{align*}
        g_j^{(L-1),1} = \sum_{i=1}^{m_{L-2}} \tilde{w}_{j,s_i}^{(L-2)} f_{s_i}^{(L-2)},
    \end{align*}
    such that 
    \begin{align*}
        &\big\|g_j^{(L-1)} - g_j^{(L-1),1}\big\|_2 \\ \nonumber
        &\leq \big\|g_j^{(L-1)} - g_{j,\theta^*}^{(L-1)}\big\|_2  + Ct_{L-2}(m_{L-2})^{\frac{1}{2}-\frac{1}{q_{L-2}}}\max_{1 \leq j \leq n_{L-2}}\big\|f_j^{(L-2)}\big\|_2 \\
        &\leq Ct_{L-2}(m_{L-2})^{\frac{1}{2}-\frac{1}{q_{L-2}}}\max_{1 \leq j \leq n_{L-2}}\big\|f_j^{(L-2)}\big\|_2, \numberthis \label{eq:g2}
    \end{align*}
    where $g_{j,\theta^*}^{(L-1)} = \arg\min_{g_{\theta} \in \mathcal{F}^{(L-2)}_{q_{L-2}}(t_{L-2})}\big\|g_j^{(L-1)} - g_{\theta}\big\|_2$.
    The last inequality is due to  $g_j^{(L-1)} \in \mathcal{F}^{(L-2)}_{q_{L-2}}(t_{L-2})$, which implies
    $\big\|g_j^{(L-1)} - g_{j,\theta^*}^{(L-1)}\big\|_2 = 0$.
    
    Let $f_j^{(L-1),1} = \sigma(g_j^{(L-1),1})$ be the approximation function for $f_j^{(L-1)}$ after one step approximation. Plugging in this into $f_T^{(1)}$, we obtain the approximation of $f_T$ after two steps as 
    \begin{align*}
        f_T^{(2)} = \sum^{m_{L-1}}_{j=1}\Tilde{w}_{1s_j}^{(L-1)} f_{s_j}^{(L-1),1},
    \end{align*}
    and the approximation error is
    \begin{align*}
        \big\|f - f_{T}^{(2)}\big\|_2 &\leq \big\|f - f_{T}^{(1)}\big\|_2 + \big\|f_{T}^{(1)} - f_{T}^{(2)}\big\|_2,
    \end{align*}
    where
    \begin{align*}
     \big\|f_{T}^{(1)} - f_{T}^{(2)}\big\|_2 &= \big\|\sum^{m_{L-1}}_{j=1}\Tilde{w}_{1s_j}^{(L-1)} (f_{s_j}^{(L-1)} -  f_{s_j}^{(L-1),1})\big\|_2 \\
     &\leq \sum^{m_{L-1}}_{j=1}|\Tilde{w}_{1s_j}^{(L-1)}|\big\|f_{s_j}^{(L-1)} - f_{s_j}^{(L-1),1}\big\|_2 \\ 
     &= \sum^{m_{L-1}}_{j=1}|\Tilde{w}_{1s_j}^{(L-1)}|\big\|\sigma(g_{s_j}^{(L-1)}) - \sigma(g_{s_j}^{(L-1),1})\big\|_2 \\
     &\leq \rho \sum^{m_{L-1}}_{j=1}|\Tilde{w}_{s_j}^{(L-1)}|\big\|g_{s_j}^{(L-1)} - g_{s_j}^{(L-1),1}\big\|_2 \\
    \textrm{\Autoref{eq:w1, eq:g2}} &\leq \rho t_{L-1}t_{L-2}C(m_{L-2})^{\frac{1}{2}-\frac{1}{q_{L-2}}}\max_{1 \leq j \leq n_{L-2}}\big\|f_j^{(L-2)}\big\|_2.
    \end{align*}
    This completes the case for $S=2$.  After $S$ steps, we have approximated $k$ steps for neurons or functions in the $(L-S+k)$-th layer for $1\leq k\leq S$. Let $g_j^{(k),s}$ be the $g_j^{(k)}$
    approximated $s$ time and $g_j^{(k),0} = g_j^{(k)}$. In particular, 
    \begin{align}
        f_j^{(k),s} = \sigma(g_j^{(k),s}), \ g_{j}^{(k),s+1} = \sum^{m_{k-1}}_{i=1}\Tilde{w}_{j,s_i}^{(k-1)}f_{s_i}^{(k-1),s}, 0\leq s \leq k <L. \label{eq:recur}
    \end{align}
    
    
    For $S \geq 3$, suppose we have finished $S-1$ steps, now we need to approximate functions in the $(L-S+1)$-th layer. 
    With the same argument as the case $S=2$, for any $1 \leq j\leq n_{L-S+1}$, we have a sparse linear approximation for $g_j^{(L-S+1)}$ as 
    \begin{align*}
        g_{j}^{(L-S+1),1} = \sum^{m_{L-S}}_{i=1}\Tilde{w}_{j,s_i}^{(L-S)}f_{s_i}^{(L-S)},
    \end{align*}
    such that 
    \begin{align*}
        &\big\|g_j^{(L-S+1)} -  g_{j}^{(L-S+1),1}\big\|_2 
        \leq Ct_{L-S}(m_{L-S})^{\frac{1}{2}-\frac{1}{q_{L-S}}}\max_{1 \leq j \leq n_{L-S}}\big\|f_j^{(L-S)}\big\|_2,
    \end{align*}
    with  $ \sum^{m_{L-S}}_{i=1}\abs{\Tilde{w}_{j,s_i}^{(L-S)}} \leq t_{L-S}$. 
    
    Therefore, the approximation error for the functions of the nodes in the $(L-S+1)^{th}$ layer after one step approximation is bounded by
    \begin{align*}
        \big\|f_j^{(L-S+1)} - f_j^{(L-S+1),1}\big\|_2 
        &\leq  \rho\big\|g_j^{(L-S+1)} -  g_{j}^{(L-S+1),1}\big\|_2 \\ 
        &\leq \rho t_{L-S}C(m_{L-S})^{\frac{1}{2}-\frac{1}{q_{L-S}}}\max_{1 \leq j \leq n_{L-S}}\big\|f_j^{(L-S)}\big\|_2. \numberthis\label{eq18}
    \end{align*}
    
    Furthermore, for $1 \leq k \leq S-2$, we have
    \begin{align*}
        & \max_{1\leq j \leq n_{L-S+k-1}} \big\|f_j^{(L-S+k+1),k+1} - f_j^{(L-S+k+1),k}\big\|_2  \\ 
        &= \max_{1\leq j \leq n_{L-S+k-1}} \big\|\sigma(g_{j}^{(L-S+k+1),k+1}) - \sigma(g_{j}^{(L-S+k+1),k})\big\|_2 \\ 
        &\leq \rho \max_{1\leq j \leq n_{L-S+k-1}} \big\|g_{j}^{(L-S+k+1),k+1} - g_{j}^{(L-S+k+1),k}\big\|_2 \\ 
        &\leq \rho \max_{1\leq j \leq n_{L-S+k-1}} \biggl[ \sum^{m_{L-S+k}}_{i=1}|\Tilde{w}_{j,s_i}^{(L-S+k)}|\big\|f_{s_i}^{(L-S+k),k} - f_{s_i}^{(L-S+k),k-1} \big\|_2 \biggr]
        \\ &\leq \rho t_{m_{L-S+k}} \max_{1\leq j \leq n_{L-S+k}} \big\|f_{j}^{(L-S+k),k} - f_{j}^{(L-S+k),k-1} \big\|_2. \numberthis \label{eq20} 
    \end{align*}
    
    Repeating \Autoref{eq20} from $k=S-2$ to $k=1$, along with \Autoref{eq18}, we have 
    \begin{align*}\label{eq:last}
        \big\|f_j^{(L-1),S-1} - f_j^{(L-1),S-2}\big\|_2 
        \leq  \rho^{S-1}t_{L-2}\ldots t_{L-S}C(m_{L-S})^{\frac{1}{2}-\frac{1}{q_{L-S}}}\max_{1 \leq j \leq n_{L-S}}\big\|f_j^{(L-S)}\big\|_2.
    \end{align*}
    Finally,
    \begin{align*}
        \big\|f_{T}^{(S)} - f_{T}^{(S-1)}\big\|_2  
        &=\big\|\sum^{m_{L-1}}_{j=1}\Tilde{w}_{1s_j}^{(L-1)} (f_{s_j}^{(L-1),S-1} -  f_{s_j}^{(L-1),S-2})\big\|_2\\ 
        &\leq \sum^{m_{L-1}}_{j=1}|\Tilde{w}_{1s_j}^{(L)}|\big\|f_{s_j}^{(L),S-1} - f_{s_j}^{(L),S-2}\big\|_2 \\
        &\leq  t_{L-1} \bnorm{f_{s_j}^{(L),S-1} - f_{s_j}^{(L),S-2}}_2 \\
        &\leq \rho^{S-1}t_{L-1}\ldots t_{L-S}C(m_{L-S})^{\frac{1}{2}-\frac{1}{q_{L-S}}}\max_{1 \leq j \leq n_{L-S}}\big\|f_j^{(L-S)}\big\|_2,
    \end{align*}
   which completes the proof by induction.

\section{Sparsity index}
    For $q\in(0,1)$ and a vector $w \in \mathbb{R}^d$, we define the \textit{sparsity index} (SI) as
    \begin{equation*}\label{eq:si}
        \SI(w) = \norm{w}_1/\norm{w}_q.
    \end{equation*}
    We show some basis properties of the sparsity index. By Jensen's inequality, we have
    \begin{align*}
        d^{-\frac{1}{q}}\norm{w}_q = \biggl(d^{-1}\sum_{i=1}^{d}\abs{w_i}^q\biggr)^{\frac{1}{q}} \leq \frac{1}{d}\sum_{i=1}^{d} \abs{w_i} = \frac{1}{d}\norm{w}_1,
    \end{align*}
    so 
    \begin{align*}
        \norm{w}_q \leq d^{\frac{1}{q}-1}\norm{w}_1.
    \end{align*}
    Furthermore, it is well-known that $\norm{w}_1 \leq \norm{w}_q$. As a result, $\SI(w) \in [d^{1-\frac{1}{q}}, 1]$, and a larger SI indicates a sparser vector. 
    
    For a near-hard sparsity scenario, let $I_m$ be the largest $m$ components of $w$, and assume $\eta$ is a constant such that $\sum_{i \notin I_m}|w_i|^q\leq \eta \sum_{i \in I_m}|w_i|^q$,
    we have
    \begin{align*}
        \norm{w}_q &= \biggl(\sum_{1 \leq i \leq d}|w_i|^q\biggr)^{\frac{1}{q}} 
        = \biggl(\sum_{i \in I_m}|w_i|^q + \sum_{i \not\in I_m}|w_i|^q\biggr)^{\frac{1}{q}} \\
        &\leq \biggl(\sum_{i \in I_m}|w_i|^q + \eta  \sum_{i \in I_m}|w_i|^q\biggr)^{\frac{1}{q}} 
        = \biggl(\sum_{i \in I_m}|w_i|^q\biggr)^{\frac{1}{q}}(1 + \eta )^{\frac{1}{q}} \\
        &\leq \biggl(\sum_{i \in I_m}|w_i|\biggr)m^{\frac{1}{q}-1} (1 + \eta )^{\frac{1}{q}} 
        \leq \norm{w}_1 m^{\frac{1}{q}-1} (1 + \eta )^{\frac{1}{q}}. 
    \end{align*}
    Rearranging the above inequality gives
    \begin{align*}
         m \geq [SI(w)]^{-q/(1-q)}(1+\eta )^{-1/(1-q)}.
    \end{align*}

\section{Proof of Theorem 4}
    We only need to prove the following counter part of the \Autoref{lem4.1}. The rest of the proof exactly follows the proof of \Autoref{thm4.1}.

    \begin{lemma}\label{lem:mag}
    Let $f$ be any function with $\big\| f \big\|_2 < \infty$. Suppose $\mathcal{F} = \{f_1, \ldots, f_M\}$ with $\max_{1 \leq j \leq M}\big\| f_j \big\|_2 < \infty$. For any $0 < q \leq 1$ and $t>0$, define the $l_{q,t}$-hull of $\mathcal{F}$ to be the class of linear combinations of functions in $\mathcal{F}$ with the $\ell_q$-constraint
    \begin{align*}
        \mathcal{F}_q(t) = \biggl\{f_\theta = \sum^m_{j=1}\theta_j f_j : \big\|\theta \big\|_q \leq t, f_j \in  \mathcal{F} \biggr\}.
    \end{align*}
    For any $1 \leq m \leq M$ and $t=\norm{w}_q$, let $f_{\theta^*} = \arg\min_{f_\theta \in \mathcal{F}_{q(t)}}\big\|f - f_{\theta}\big\|_2 =  \sum^m_{j=1}w_j f_j$, $J_m = \{j: \abs{w_j} > tm^{-1/q}\}$, and $\hat{f} = \sum_{j \in J_m}w_j f_j$. Then, we have
    \begin{align} 
        \big\|f - \hat{f}\big\|_2 \leq \big\|f - f_{\theta^*}\big\|_2 + tm^{1-1/q}\max_{1 \leq j \leq M}\big\|f_j\big\|_2.
    \end{align}
    Additionally, $\sum_{j \in J_m}\abs{w_j} \leq t$ and the cardinality of $J_m$ is no more than $m$.
    \end{lemma}
    \begin{proof}
        First, we know $\sum_{j \in J_m}\abs{w_j} \leq \norm{w}_1 \leq \norm{w}_q = t$, thus
        the cardinality of $J_m$ is no more than $m$ since 
        \begin{align*}
            \sum_{j \in J_m}t^q/m \leq \sum_{j \in J_m}\abs{w_j}^q \leq \norm{w}_q^q = t^q.
        \end{align*}
        Second, we have
        \begin{align*}
            \sum_{j \notin J_m} \abs{w_j} \leq \sum_{j \notin J_m} \abs{w_j}^q(tm^{-1/q})^{1-q} \leq \norm{w}_q^q(tm^{-1/q})^{1-q} = tm^{1-1/q},
        \end{align*}
        hence
        \begin{align*}
            \bnorm{f_{\theta^*} - \hat{f}}_2 = \biggl\|\sum_{j \notin J_m} w_jf_j\biggr\|_2 \leq \sum_{j \notin J_m} \abs{w_j} \max_{1 \leq j \leq M}\big\|f_j\big\|_2 \leq tm^{1-1/q}\max_{1 \leq j \leq M}\big\|f_j\big\|_2.
        \end{align*}
        We finish the proof by plugging in the above inequality into the following triangle inequality
        \begin{align*}
            \big\|f - \hat{f}\big\|_2 \leq \big\|f - f_{\theta^*}\big\|_2 + \big\| f_{\theta^*}- \hat{f}\big\|_2.
        \end{align*}
    \end{proof}

\bibliographystyle{plainnat}
\bibliography{ref}

\end{document}